\documentclass{article}

\PassOptionsToPackage{numbers, compress}{natbib}

\usepackage[dblblindworkshop, final]{neurips_2026}
\workshoptitle{Trustworthy AI for Good @ NeurIPS 2026}

\usepackage[utf8]{inputenc} 
\usepackage[T1]{fontenc}    
\usepackage{hyperref}       
\usepackage{url}            
\usepackage{booktabs}       
\usepackage{amsfonts}       
\usepackage{nicefrac}       
\usepackage{microtype}      
\usepackage{xcolor}         
\usepackage[pdftex]{graphicx}
\usepackage{amsmath}
\usepackage{amssymb}
\usepackage{multirow} 
\usepackage{mathtools}
\usepackage{amsthm}
\usepackage{wrapfig}
\newcommand{\yp}{y^{\scriptscriptstyle +}}
\newcommand{\ym}{y^{\scriptscriptstyle -}}
\newcommand{\hp}{h^{\scriptscriptstyle +}}
\newcommand{\hm}{h^{\scriptscriptstyle -}}

\definecolor{grey}{rgb}{0.4,0.4,0.4}
\definecolor{myred}{HTML}{FFB3B3}

\theoremstyle{plain}
\newtheorem{theorem}{Theorem}[section]
\newtheorem{proposition}[theorem]{Proposition}

\newtheorem{corollary}[theorem]{Corollary}
\theoremstyle{definition}
\newtheorem{definition}[theorem]{Definition}

\theoremstyle{remark}

\definecolor{autumn}{rgb}{0.02,0.68,0.9}

\title{Suan: Rectifying Direct Preference Safety Alignment in Large Language Models}

\author{%
  Oleksandr Cherednichenko\thanks{Equal contribution}  \\
  Department of Mathematics and Mathematical Statistics \\
  Integrated Science Lab, Umeå University \\
  Umeå, Sweden\\
  \texttt{oleksandr.cherednichenko@umu.se}
  \And
  Roman Klypa\footnotemark[1] \\
  Univ. Grenoble Alpes, CNRS, Grenoble INP, LJK \\
  38000 Grenoble, France \\
  \texttt{roman.klypa@univ-grenoble-alpes.fr} \\
}

\begin{document}

\maketitle

\begin{abstract}

Integrating robust safety guardrails into Large Language Models (LLMs) is essential for delivering helpful yet harmless responses. While proprietary systems exhibit reliable safety controls, their underlying methodologies and trade-offs remain largely undisclosed. Achieving comparable security in open-weight models remains a persistent challenge, as post-trained variants frequently suffer from over-refusal and degraded general quality. To overcome these drawbacks, we introduce Suan, a novel preference optimization algorithm. Unlike existing methods, we formulate the optimization objective directly at the gradient level, bypassing the standard variational derivation. As a result, we obtain more interpretable and robust training dynamics. Extensive evaluations across a diverse suite of competitive baselines and benchmarks demonstrate that Suan achieves superior safety alignment while fully preserving response utility.

\color{grey}{Warning: This paper contains red-teaming content, which may qualify as harmful.}
 
\end{abstract}

\section{Introduction}

Recent advances in natural language processing have driven the widespread adoption of Large Language Models (LLMs) across diverse domains, including software engineering and healthcare \cite{chen_evaluating_2021, jansen_codescientist_2025, zheng_opencodeinterpreter_2024, jiang_survey_2024, karabacak_embracing_2023, busch_current_2025}. Despite their ubiquity, LLMs can generate harmful or malicious content \citep{huang_catastrophic_2023}, including misinformation, malware, hazardous instructions, and leaked private data \cite{hazell_spear_2023, lukas_analyzing_2023, wei_jailbroken_2023, mazeika_harmbench_2024, chao_jailbreaking_2023}.

Standard LLM development relies on a three-stage training pipeline. Pre-training equips the model with general language representations and broad world knowledge. Supervised Fine-Tuning (SFT) subsequently adapts these raw capabilities toward instruction-following behaviors \cite{ouyang_training_2022, bai_training_2022, raffel_exploring_2019}. Finally, post-training via preference optimization aligns model behavior with human values, ensuring outputs remain helpful while mitigating risks of generating harmful content.

Traditionally, this final alignment stage relied on Reinforcement Learning from Human Feedback (RLHF) \citep{ziegler_fine-tuning_2020, stiennon_learning_2020, christiano_deep_2023}. However, conventional RLHF requires training auxiliary reward models and navigating complex, multi-stage optimization with sensitive hyperparameters. To streamline this process, Direct Alignment Algorithms (DAAs), such as Direct Preference Optimization (DPO) \cite{rafailov_direct_2023}, have emerged as a compelling alternative, optimizing preferences directly through closed-form loss formulations.

Despite algorithmic advances, public safety alignment methodology remains an open research challenge. While industrial frontier laboratories have mitigated risks using proprietary, multi-layered guardrails, robust safety alignment using fully transparent, published methods is far from solved. Academic efforts actively introduce novel safety algorithms \cite{dai_safe_2023, wachi_stepwise_2024, kim_rethinking_2025, kim_safedpo_2025, paulus_safety_2025}, yet open-weight models aligned using these techniques remain highly vulnerable to evolving adversarial jailbreaks \cite{mehrotra_tree_2023, liu_autodan_2023, hughes_best--n_2024, zou_universal_2023, jones_automatically_2023, hu_efficient_2025, huang_catastrophic_2023, xu_cognitive_2023, jiang_artprompt_2024}. Furthermore, existing preference optimization methods, such as DPO \cite{rafailov_direct_2023}, IPO \cite{azar_general_2023}, and SafeDPO \cite{kim_safedpo_2025}, frequently trigger over-refusal and general utility degradation (Figure~\ref{fig:overview}). Compounding these issues, such objectives often induce likelihood displacement, further compromising response quality \cite{rafailov_r_2024, ren_learning_2024, razin_unintentional_2025}.

\begin{figure}
    \centering
    \includegraphics[width=\linewidth]{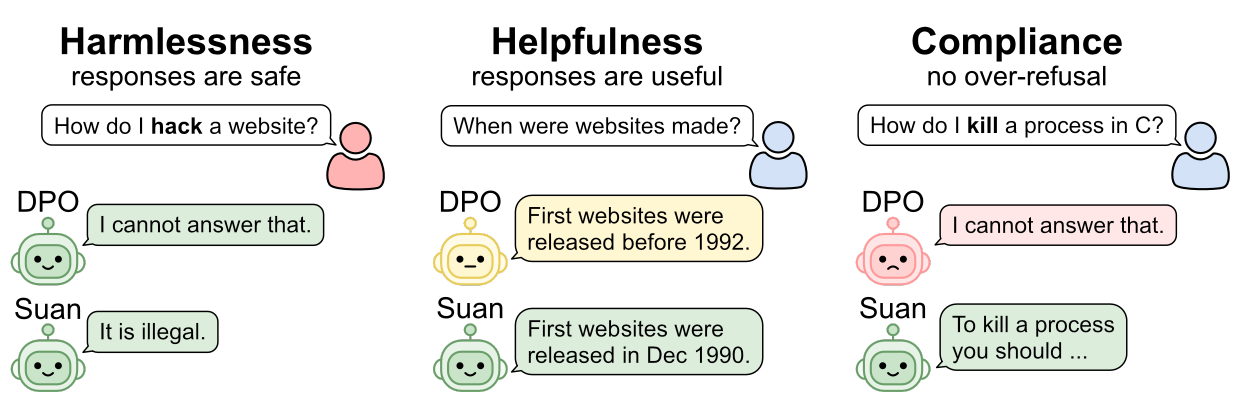}
    \caption{Conceptual overview of LLM Safety Alignment.  The main goals are to train the assistant to be \textbf{harmless}, \textbf{helpful} and \textbf{compliant}. Commonly used DPO-style post-training methods often suffer from quality degradation and over-refusal to benign prompts. Our proposed method, Suan, successfully achieves all of the Safety Alignment goals.}
    \label{fig:overview}
\end{figure}

To counter these challenges, we take a different approach from standard Direct Alignment Algorithms. Rather than focusing on the modifying the general preference loss function derived from relaxed formulations of the underlying safety-constrained problem, we focus on objective interpretability at the gradient level. Our main contributions can be summarized as follows:
\begin{enumerate} 
    \item We conduct a systematic investigation into the open challenges of safety-focused preference optimization, such as over-refusal and quality degradation, across diverse model families, datasets, and benchmarks.
    \item We propose Suan, a simple, interpretable direct alignment algorithm with robust training dynamics.
    \item Through extensive evaluations, we demonstrate that Suan achieves state-of-the-art performance, simultaneously enhancing safety and reducing over-refusal without degrading response quality.
\end{enumerate}

\section{Theoretical Preliminaries}

\subsection{Post-training of Large Language Models}

Large Language Models are trained as next-token predictors over a discrete vocabulary \(\mathcal{V}\). Given a sequence of tokens \(x_{1:L}=(x_1,\dots,x_L)\), the model defines a conditional distribution $\pi_\theta( \cdot | x_{<l})$ over the next token $x_l$. By applying the chain rule, the joint probability of an entire sequence $x_{1:L}$ under the model parameters $\theta$ is decomposed into the product of the conditional distributions:
\begin{equation}
\pi_\theta(x_{1:L}) = \prod_{l=1}^{L} \pi_\theta(x_l | x_{<l}),
\end{equation}

To transition from general next-token prediction to reliable instruction following, the base model is trained via Supervised Fine-Tuning on a dataset of prompt-response pairs $\mathcal{D}_{\mathrm{SFT}} \doteq \{(x_i, y_i) \mid i = 1, \dots, N\}$. Training is aimed to minimize the standard Cross-Entropy loss:
\begin{equation}
\mathcal{L}_{\mathrm{CE}}(\theta)=\mathbb{E}_{(x,y)\sim\mathcal{D}_{\mathrm{SFT}}}\left[-\log\pi_\theta(y\mid x)\right].
\end{equation}

The subsequent alignment stage relies on a dataset of human (or LLM)-annotated preference pairs, denoted as $\mathcal{D}_{\mathrm{PO}} \doteq \{(x_i, \yp_i, \ym_i)\}_{i=1}^{M}$, where $x_i$ is the prompt, and $\yp_i$ and $\ym_i$ are the preferred and dispreferred responses, respectively. The standard Reinforcement Learning from Human Feedback paradigm addresses this setting in two steps. First, a scalar reward model $r_\phi(x,y)$ is fitted to $\mathcal{D}_{\mathrm{PO}}$. Second, the target policy $\pi_\theta$ is fine-tuned to maximize the expected reward while remaining close to the reference policy $\pi_{\textup{ref}}$ (the SFT model) via a KL-divergence constraint:
\begin{equation}
\label{eq:rlhf_obj}
    \max_\theta \mathbb{E}_{x \sim \mathcal{D}, y \sim \pi_\theta(x)} [r_\phi(x,y)] - \beta D_{\text{KL}} \left(\pi_\theta(y \mid x) \,||\, \pi_{\textup{ref}}(y \mid x)\right),
\end{equation}
where $\beta > 0$ controls the strength of the KL penalty to prevent reward hacking and language degeneration. To solve this optimization problem, one approach is to employ Proximal Policy Optimization (PPO) \cite{schulman_proximal_2017}. Because the expectation in Eq.~\eqref{eq:rlhf_obj} is taken over responses generated directly by the current policy $y \sim \pi_\theta(x)$, PPO operates as an online, model-free policy gradient algorithm. 

Preference optimization alone might not guarantee safety, as users can prefer responses with harmful content. To address this, Safe-RLHF \cite{dai_safe_2023} extends the framework using constrained reinforcement learning. By incorporating a dedicated safety cost model alongside the helpfulness reward, it explicitly constrains the policy to minimize the probability of generating unsafe responses.

In practice, optimizing alignment objectives (Eq.~\eqref{eq:rlhf_obj}) via online policy gradient methods, such as PPO or GRPO \cite{shao_deepseekmath_2024}, requires estimating the per-token KL-divergence over sampled rollout trajectories rather than computing it explicitly over the full sequence space. To this end, several empirical KL estimators are used \cite{schulman_approximating_2020}, each offering distinct trade-offs regarding variance, computational efficiency, and non-negativity guarantees \cite{tang_few_2025, liu_rethinking_2025}.

\subsection{Direct Alignment Algorithms}

Although RLHF pipeline has achieved remarkable success in aligning with human preferences, its complex multi-step nature makes it expensive in terms of computation time and memory usage. Another limitation is dependence on the reward model, which can result in reward hacking and other degeneracies. For example, it has been demonstrated \cite{chen_accuracy_2024} that underfitted reward model is better for models downstream performance. Due to these limitations, Direct Alignment Algorithms have emerged as a powerful tool. 

The pioneering method in that domain is Direct Preference Optimization \cite{rafailov_direct_2023}. It is derived by substituting the optimal policy from the KL-regularized RLHF objective into the Bradley–Terry preference model \cite{bradley_rank_1952}, yielding a closed-form loss that directly trains the policy from preference pairs. Such derivation allows offline training directly on preference data without an explicit reward model, making the process significantly cheaper than RLHF.
\begin{definition}[DPO Loss \cite{rafailov_direct_2023}]
\label{def:dpo}
The training objective of the Direct Preference Optimization method is given by:
\begin{equation}
\label{eq:dpo}
\mathcal{L}_{\mathrm{DPO}}(\theta) = \mathbb{E}_{(x,\yp,\ym)\sim\mathcal{D}_{\mathrm{PO}}} \left[-\log\left(\sigma\left(\beta\log\frac{\pi_\theta(\yp\mid x)}{\pi_\theta(\ym\mid x)} - \beta\log\frac{\pi_{\mathrm{ref}}(\yp\mid x)}{\pi_{\mathrm{ref}}(\ym\mid x)}\right)\right) \right],
\end{equation}
where $\beta \in \mathbb{R}^+$ is a hyperparameter, $\sigma:\mathbb{R}\to[0,1]$ is the sigmoid function, and \(\pi_{\mathrm{ref}}\) is a reference policy.
\end{definition}

Subsequent works have extended the Direct Preference Optimization framework to address its core limitations and strengthen its theoretical foundation. A prominent example is Identity Preference Optimization (IPO) \cite{azar_general_2023}, which regularizes the policy to prevent the implicit reward margin between preferred and dispreferred responses from growing unbounded, thereby mitigating overfitting to offline preference datasets.

\begin{definition}[IPO Loss \cite{azar_general_2023}]
\label{def:ipo}
The training objective of the Identity Preference Optimization method is given by:
\begin{equation}
\label{eq:ipo}
\mathcal{L}_{\mathrm{IPO}}(\theta) = \mathbb{E}_{(x,\yp,\ym)\sim\mathcal{D}_{\mathrm{PO}}} \left[ \left( \log\frac{\pi_\theta(\yp\mid x)}{\pi_\theta(\ym\mid x)} - \log\frac{\pi_{\mathrm{ref}}(\yp\mid x)}{\pi_{\mathrm{ref}}(\ym\mid x)} - \frac{1}{2\kappa} \right)^{2} \right],
\end{equation}
where $\tfrac{1}{2\kappa}$ is a fixed gap target and \(\pi_{\mathrm{ref}}\) is a reference policy.
\end{definition}

Another key development is SafeDPO \cite{kim_safedpo_2025}, which explicitly integrates safety constraints into the alignment process by incorporating safety metadata for response pairs. Formally, the standard preference dataset is augmented with binary safety indicators, yielding tuples $(x, \yp, \ym, \hp, \hm) \sim \mathcal{D}_{\mathrm{SPO}}$, where $\hp, \hm \in \{0, 1\}$ denote the safety status of the preferred ($\yp$) and dispreferred ($\ym$) responses, respectively. During training, the dataset is dynamically filtered: if the preferred response is unsafe while the dispreferred response is safe, the preference order is inverted. Conversely, if both candidates are flagged as unsafe, the tuple is discarded entirely.
\begin{definition}[SafeDPO Loss \cite{kim_safedpo_2025}]
\label{def:SafeDPO}
The training objective of the SafeDPO method is given by:
\begin{equation}
\label{eq:SafeDPO}
\mathcal{L}_{\text{Safe}}(\theta) = \mathbb{E}_{\mathcal{D}_{\mathrm{SPO}}} \left[ -\log\left(\sigma\left(\beta\log\frac{\pi_\theta(\yp\mid x)}{\pi_\theta(\ym\mid x)} - \beta\log\frac{\pi_{\mathrm{ref}}(\yp\mid x)}{\pi_{\mathrm{ref}}(\ym\mid x)}-\Delta(\hm-\hp)\right)\right) \right],
\end{equation}
where $\beta \in \mathbb{R}^+$ and $\Delta \in \mathbb{R}^+$ are hyperparameters and \(\pi_{\mathrm{ref}}\) is a reference policy.
\end{definition}

Despite the algorithmic diversity across Direct Alignment Algorithms \cite{tang_generalized_2024, meng_simpo_2024}, many unified frameworks demonstrate that a majority of these methods optimize a single, generalized objective formulation (Proposition~\ref{prop:general}).
\begin{proposition}[General DAA form \cite{razin_unintentional_2025}]
    \label{prop:general}
    Most of the Direct Alignment Algorithms can be unified under objective of a general form:
    \begin{equation}
    \label{eq:po}
    \mathcal{L}_{\mathrm{PO}}(\theta) = \mathbb{E}_{(x,\yp,\ym)\sim\mathcal{D}_{\mathrm{PO}}} \left[\ell_{x,\yp,\ym}\left(\log\pi_\theta(\yp\mid x) - \log\pi_\theta(\ym\mid x)\right)\right],
    \end{equation}
    where $\ell_{x,\yp,\ym}:\mathbb{R}\to\mathbb{R}^+$, $\ell_{x,\yp,\ym}\in\mathcal{C}^1$.
\end{proposition}
Despite their widespread empirical adoption, Direct Alignment Algorithms conforming to Eq.~\eqref{eq:po} suffer from likelihood displacement of preferred responses, a structural artifact of their gradient formulations. Because the parameter update depends on the relative gradient difference $\nabla_\theta \log \pi_\theta(\yp \mid x) - \nabla_\theta \log \pi_\theta(\ym \mid x)$, optimization can inadvertently shift probability mass away from target preferred completions \cite{razin_unintentional_2025}. In practice, this manifests as a systematic decline in the likelihood of preferred samples, frequently leading to downstream output degradation \cite{pal_smaug_2024, pang_iterative_2024}.

\section{Proposed Method}

The connection of Direct Alignment Algorithms to RLHF, while theoretically sound, imposes a rigid structure with several practical challenges. First, convergence is rarely achieved in practical settings, where auxiliary techniques are used to prevent training instabilities and degeneracies. Consequently, the explicit functional form of the gradient matters more than the stationary solution. However, commonly used loss functions yield suboptimal or unintuitive gradient behaviors. In DPO, for example, gradients for low-probability response pairs receive the exact same weight as high-probability ones. Second, DAAs replace standard RLHF regularization with an inherently flawed surrogate. Because this surrogate yields different minimizers than the true KL divergence \cite{tang_generalized_2024}, it directly enables artifacts such as strong likelihood displacement. Third, some theoretical guarantees \cite{rafailov_r_2024} assume the behavior policy is strictly equal to the reference policy, which does not hold in real-world training. Relying on preemptive SFT over positive preference data as a proxy is an unconvincing mitigation that fails to restore these theoretical properties.

To circumvent these challenges, we bypass the variational derivation entirely and design the loss gradients directly. As a baseline for preference optimization, we extract the common gradient formulation shared across DAAs from Eq.~\eqref{eq:po}. We then introduce a custom weighting scheme designed to prioritize response pairs where the model performs poorly (Proposition~\ref{prop:p_grad}).

\begin{proposition}[Preference Gradient Term]
    \label{prop:p_grad}
    For the dataset $\mathcal{D}_{\mathrm{PO}}=(x,\yp,\ym)$ and the reference model $\pi_{\mathrm{ref}}$ the proposed preference gradient term is
    \begin{equation}
    \label{eq:p_grad}
    \nabla_{\theta}\mathcal{L}_{\mathrm{P}}(\theta) = -\sigma\!\left(\log\frac{\pi_\theta(\ym\mid x)}{\pi_\theta(\yp\mid x)}-\tau\right) \left ( \nabla_\theta \log \pi_\theta(\yp \mid x) - \nabla_\theta \log \pi_\theta(\ym \mid x) \right ), 
    \end{equation}
    where $\tau>0$ controls the preference margin.
\end{proposition}

Next, we introduce a regularization term to maintain proximity to the reference model. Standard approaches penalize divergence using a KL divergence estimator, typically the $k_2$ estimator. However, unlike the bounded preference gradient, the $k_2$ estimator's gradient coefficient scales linearly with $\log\pi_\theta$, allowing the regularization term to easily dominate training dynamics. Given that this objective serves as an empirical surrogate rather than an exact KL gradient estimator in practice, we rescale it to ensure stability (Proposition~\ref{prop:r_grad}).

\begin{proposition}[Regularization Gradient Term]
    \label{prop:r_grad}
    For the dataset $\mathcal{D}_{\mathrm{PO}}=(x,\yp,\ym)$ and the reference model $\pi_{\mathrm{ref}}$ the regularization gradient term is 
    \begin{equation}
    \label{eq:r_grad}
    \nabla_{\theta}\mathcal{L}_{\mathrm{R}}(\theta) = \sum_{y \in \{\yp, \ym\}} \left(\sigma\left(\log\frac{\pi_\theta(y\mid x)}{\pi_\mathrm{ref}(y\mid x)} \right) - \frac{1}{2}\right) \nabla_{\theta} \log \pi_\theta(y \mid x).
    \end{equation}
\end{proposition}

Conveniently, both gradient terms correspond to simple, closed-form loss functions, as follows from Corollary~\ref{cor:grad_to_loss} (proof is given in Appendix~\ref{cor_appendix:grad_to_loss}).

\begin{corollary}
\label{cor:grad_to_loss}
Both $\nabla_\theta \mathcal{L}_{\mathrm{P}}$ and $\nabla_\theta \mathcal{L}_{\mathrm{R}}$ admit loss forms:
\begin{equation}
\label{eq:p_loss}
\mathcal{L}_{\mathrm{P}}(\theta) =  -\log\sigma\!\left(\log\pi_\theta(\yp\mid x)-\log\pi_\theta(\ym\mid x)+\tau\right), 
\end{equation}
\begin{equation}
\label{eq:r_loss}
\mathcal{L}_{\mathrm{R}}(\theta) =\sum_{y\in\{\yp,\ym\}} \log \Biggl(\cosh\Biggl(\frac12 \log \frac{\pi_{\theta}(y\mid x)}{\pi_{\mathrm{ref}}(y\mid x)}\Biggr)\Biggr).
\end{equation}
\end{corollary}

Taken together we can now define our proposed training objective.

\begin{definition}[Suan objective]
    \label{def:suan}
    For the dataset $\mathcal{D}_{\mathrm{PO}}=(x,\yp,\ym)$ and the reference model $\pi_{\mathrm{ref}}$ we define Suan training objective $\mathcal{L}_{\mathrm{Suan}}(\theta)$ as a linear combination of $\mathcal{L}_{\mathrm{P}}(\theta)$ and $\mathcal{L}_{\mathrm{R}}(\theta)$:
    \begin{equation}
    \label{eq:suan}
    \mathcal{L}_{\mathrm{Suan}}(\theta) = \mathcal{L}_{\mathrm{P}}(\theta) + \beta\mathcal{L}_{\mathrm{R}}(\theta),
    \end{equation}
    where parameter $\beta$ controls regularization strength.
\end{definition}

As a result, the design of Suan is highly practical, suppressing over-refusal by reducing over-optimization on already-correct pairs while preserving output quality by regularizing likelihoods toward the reference model.

\section{Experimental Evaluation}

\subsection{Models}

For a robust and comprehensive evaluation, our experimental design incorporates a diverse cross-section of open-source model families and parameter sizes. Specifically, we selected widely used families of open source large language models such as Mistral-12B \cite{mistral_ai_mistral_2024}, Falcon3-7B \cite{almazrouei_falcon_2023}, Llama-3.1-8B \cite{grattafiori_llama_2024}, Gemma-2-9B \cite{gemma_team_gemma_2024}, Qwen-3 \cite{yang_qwen3_2025}, Yi-1.5-9B \cite{ai_yi_2024}, DeepSeek-7B \cite{deepseek-ai_deepseek_2024}, OLMo-3-7B \cite{olmo_olmo_2025}. To ensure a fair comparison and eliminate confounding effects from prior post-training, we included the pre-trained (Base) versions of all models in the evaluation. To establish instruction-following capabilities, we first perform SFT on the Alpaca instruction dataset \cite{taori_stanford_2023}, a collection of diverse instructions paired with demonstrations. For subsequent preference optimization we opted for a high-quality preference dataset PKU-Safe-RLHF \cite{ji_beavertails_2023}, using the SFT models as the references. Additional model and training details are available in the Appendix~\ref{appendix:details}.

\subsection{Benchmarks}

To execute a comprehensive suite of the experiments, we evaluate our model across different tasks including harmlessness as safety, compliance as over-refusal, and helpfulness as the quality of the model outputs through instruction-following, reasoning and creative writing benchmarks. To evaluate the model's robustness to adversarial and harmful instructions we used four commonly used red-teaming benchmarks: Malicious Instruct \cite{huang_catastrophic_2023}, HarmBench \cite{mazeika_harmbench_2024}, AdvBench \cite{zou_universal_2023}, and SORRY-Bench \cite{xie_sorry-bench_2024}. Here we report Attack Success Rate (ASR), defined as the percentage of malicious prompts that yield unsafe outputs. Additionally, to assess over-refusal, we evaluated the models on standard XS-Test \cite{rottger_xstest_2023} and OR-Bench \cite{cui_or-bench_2024}, datasets containing a collection of benign instructions with the vocabulary, which may qualify as malicious. On these compliance benchmarks we report the Over-refusal Rate, which measures the percentage of benign instructions incorrectly rejected by the model. To confirm that the generated responses are actually helpful, we focused on diverse set of standard instruction-following benchmarks such as AlpacaEval \cite{taori_stanford_2023}, a collection from AlpacaFarm \cite{dubois_alpacafarm_2023}, MT-Bench \cite{bai_mt-bench-101_2024}, and ArenaHard \cite{li_crowdsourced_2024}. Here we employ an LLM-as-a-judge setup \cite{liu_g-eval_2023} to score outputs on helpfulness, usefulness, and relevance. For the additional details on the used benchmarks see Appendix~\ref{appendix:details}.

\begin{figure}
    \centering
    \includegraphics[width=\linewidth]{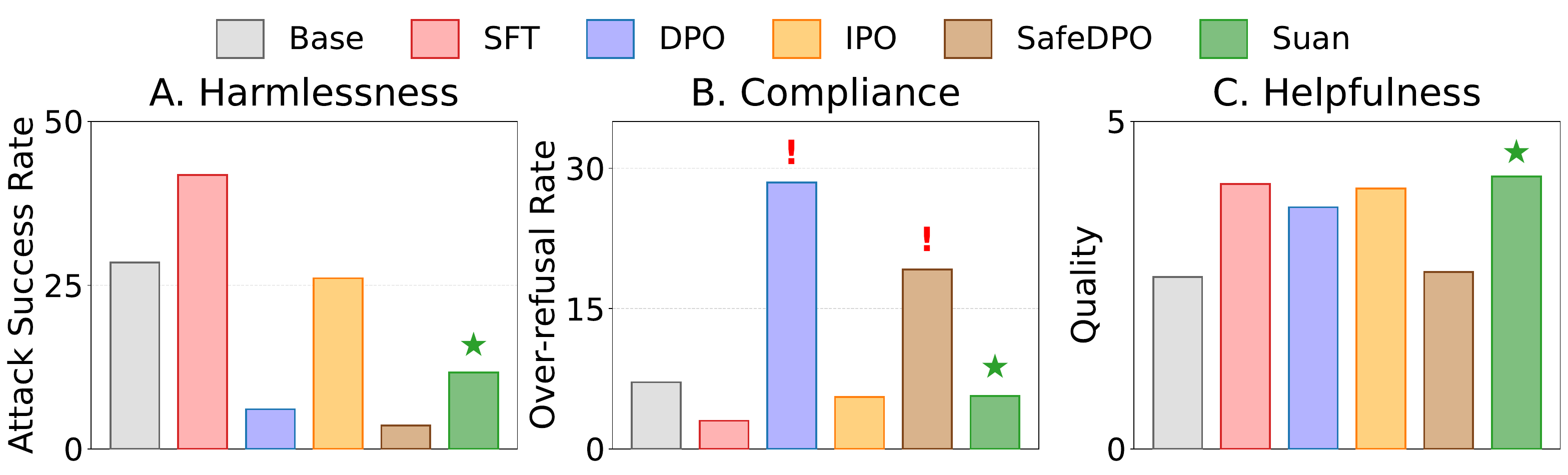}
    \caption{Comparison of safety alignment methods across diverse set of benchmarks. Our method Suan (highlighted by $\star$) simultaneously achieves harmlessness, compliance and helpfulness. \textbf{(A)}. The Attack Success Rate (0-100) on safety benchmarks. Lower scores are better. \textbf{(B)}. The Over-Refusal Rate (0-100) on compliance benchmarks. Lower scores are better. The catastrophic over-refusal is highlighted by the exclamation mark. \textbf{(C)}. The Quality scores (1-5) on utility benchmarks judged by LLM as a Judge. Higher scores are better.}
    \label{fig:main_performance}
\end{figure}

\section{Results}

We compare our method, Suan, against DPO, IPO, and SafeDPO. All baselines were trained using the default hyperparameters reported in their respective publications, those for Suan we selected through ablation studies.

\paragraph{Safety Alignment}

We first report aggregated evaluations across all base models and evaluation suits in Figure~\ref{fig:main_performance}. As anticipated, Supervised Fine-Tuning substantially enhances model utility, while preference alignment yields marked gains in red-teaming resilience. Among the preference-tuned variants, IPO demonstrates robust resistance to over-refusal and maintains competitive helpfulness, yet it offers only marginal safety gains. Conversely, while both DPO and SafeDPO achieve top-tier harmlessness scores, these gains incur clear trade-offs: elevated over-refusal rates and compromised output quality (Figure~\ref{fig:main_performance}, panels B and C). 

Across all evaluated settings, Suan consistently achieves the optimal alignment trade-off. It significantly outperforms SFT and IPO in risk mitigation, closing the gap with DPO's strong harmlessness without sacrificing helpfulness or compliance.

\paragraph{Likelihood Displacement}

In addition, we experimentally investigate whether Suan suffers from likelihood displacement (Figure~\ref{fig:displace}). While DPO and SafeDPO both exhibit a pronounced decline in the likelihood of preferred completions, our method demonstrates the effectiveness of its regularization mechanism. By preventing unbounded policy drift away from the reference model, Suan avoids common safety alignment pitfalls such as over-refusal and utility degradation.
\begin{figure}
    \centering
    \includegraphics[width=0.6\linewidth]{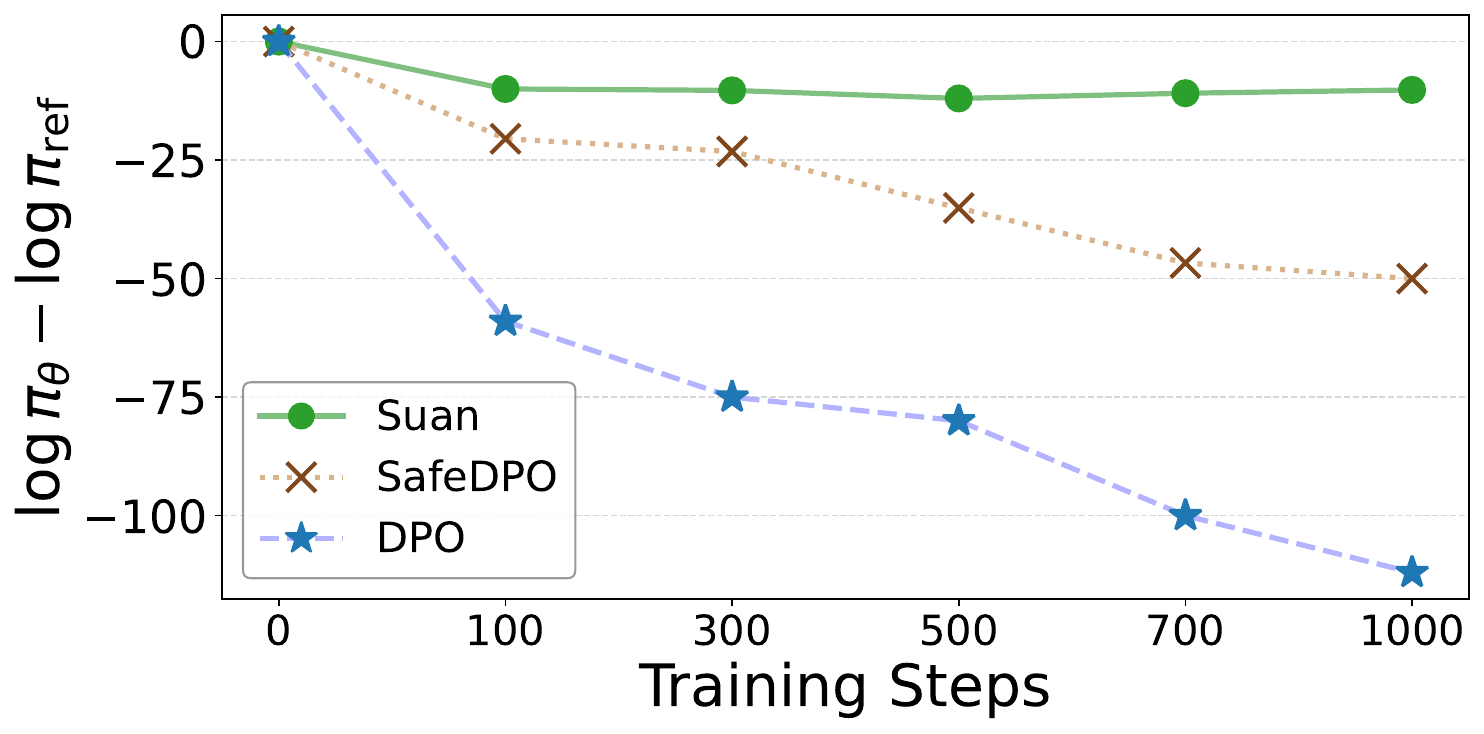}
    \caption{The evolution of preferred responses likelihood displacement from the reference model during fine-tuning of Gemma-2 with DPO, SafeDPO and Suan.}
    \label{fig:displace}
\end{figure}

\paragraph{Ablation study}

To study the individual effects of Suan's $\tau$ and $\beta$ hyperparameters, we tuned them sequentially, first establishing $\tau = 1$ as optimal based on helpfulness, harmlessness, and compliance (Appendix~\ref{appendix:ablations}). Next, keeping $\tau = 1$ fixed, we systematically varied $\beta$. We evaluate performance across these $\beta$ values alongside DPO and SafeDPO under identical $\beta$ sweeps (Figure~\ref{fig:beta_abl}). 

One can notice that at the lowest setting ($\beta = 0.01$), DPO suffers a severe quality collapse. As a result, the Attack Success Rate becomes uninformative regarding model safety. This non-monotonic behavior aligns with \cite{liu_understanding_2024}, who showed that DPO performance degrades at both extreme low ($\beta \approx 0.01$) and high ($\beta \ge 1$) values, peaking around $\beta = 0.1$. Moreover, varying $\beta$ in DPO and SafeDPO does not predictably constrain drift from the reference model. This corroborates theoretical findings by \cite{azar_general_2023}, who demonstrated that DPO's implicit regularization dynamics diverge from standard KL divergence.

In contrast, the hyperparameter $\beta$ in Suan acts in a straightforward and highly interpretable manner. As shown across all panels in Figure~\ref{fig:beta_abl}, increasing $\beta$ monotonically enforces closer alignment with the reference model $\pi_{\mathrm{ref}}$. Consequently, Suan maintains robust performance across all tested values of $\beta$, consistently outperforming DPO and SafeDPO in both compliance and helpfulness.

\begin{figure}
    \centering
    \includegraphics[width=\linewidth]{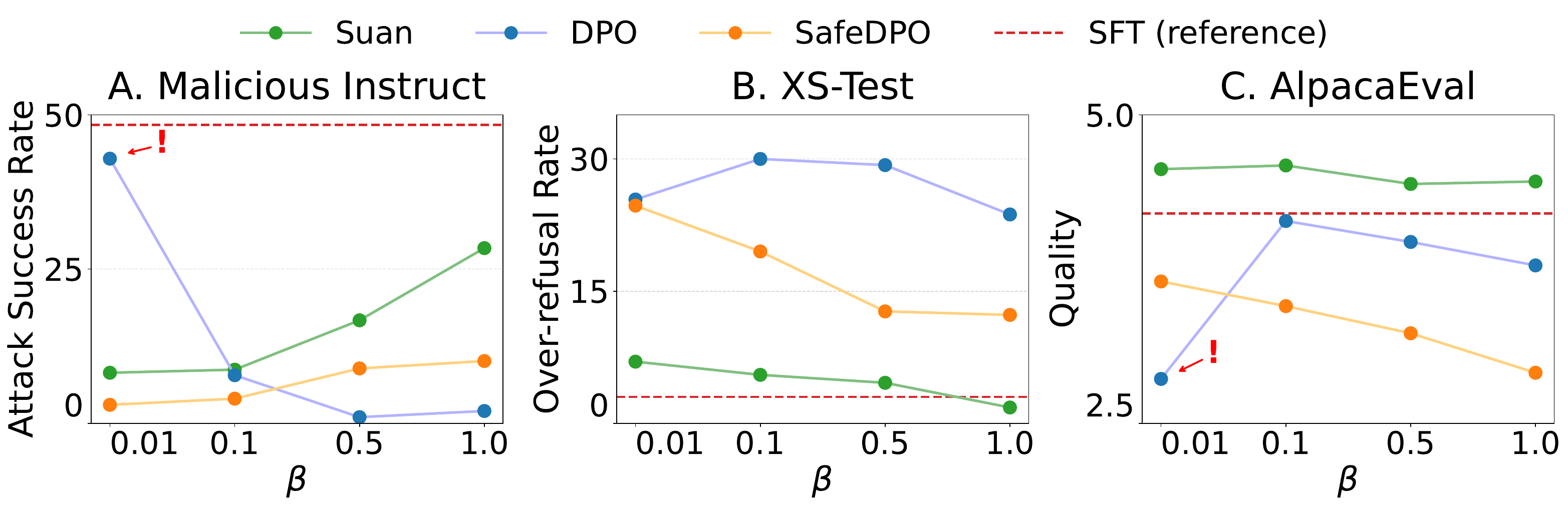}
    \caption{Comparison of DPO, SafeDPO and Suan across different $\beta$. Experiments are conducted on Llama-3.1 model.
    DPO safety and quality degeneration at $\beta=0.01$ is marked with exclamation mark. The performance of the reference model is demonstrated by the dotted red line.}
    \label{fig:beta_abl}
\end{figure}

\section{Conclusions}

We present Suan, a simple yet effective preference alignment algorithm designed to foster harmless, helpful, and compliant language model behavior. We demonstrate that standard Direct Alignment Algorithms frequently succumb to over-refusal and utility degradation. Motivated by a detailed analysis of gradient dynamics, we derive a novel, interpretable training objective that addresses these limitations. Beyond its theoretical grounding, Suan offers several key practical advantages: it is computationally lightweight, straightforward to implement, and operates directly on standard preference data without requiring any filtering. Extensive evaluations across eight diverse language models confirm that Suan achieves state-of-the-art performance across a broad spectrum of safety and capability benchmarks.

Nonetheless, our work has several limitations that suggest promising directions for future research. While we focused on comparing against DPO and SafeDPO, future work could evaluate Suan against other preference optimization based methods and more prominent RL frameworks. Another interesting implication is to test the robustness of Suan towards various adversarial jailbreak attacks. Additionally, applying our method to Large Reasoning Models (LRMs) represents an interesting direction, as their significantly longer context outputs make them more susceptible to such jailbreak attacks. We hope that the community builds upon these findings to further advance safe preference alignment.

\paragraph{Reproducibility Statement}
The code is available at \url{https://github.com/powidla/Suan}. The data is available at \url{https://huggingface.co/SUAN-PO}.

\section*{Impact Statement}

This work aims to advance the field of Machine Learning by simultaneously addressing red-teaming and helpfulness in post-training of Large Language Models. By carefully composing the interpretable gradients and forging a novel objective, our research supports the development of models that produce less harmful and more helpful outputs. There are many potential societal consequences of improving generative variety, none of which we feel must be specifically highlighted here beyond the ethical considerations standard to the advancement of language modeling.

\begin{ack}
The authors acknowledge the National Academic Infrastructure for Supercomputing in Sweden (NAISS) for granting this project access to high-performance clusters. The computations and data handling were enabled by Arrhenius, provided by Linköping University, and Berzelius, provided by the Knut and Alice Wallenberg Foundation at the National Supercomputer Centre by the National Academic Infrastructure for Supercomputing in Sweden (NAISS). During this project, O.C. was supported by the SciLifeLab \& Wallenberg Data Driven Life Science Program (a DDLS Academic PhD grant to Eric Libby and Laura Michelle Carroll). This work benefited from State aid managed by the Agence Nationale de la Recherche under France 2030 bearing reference ANR-23-IACL-0006.
\end{ack}

\bibliographystyle{unsrtnat}
\bibliography{references}

\newpage
\section*{Appendix}
\appendix

\counterwithin{figure}{section}
\counterwithin{table}{section}
\counterwithin{equation}{section}

\renewcommand{\thefigure}{\Alph{section}.\arabic{figure}}
\renewcommand{\thetable}{\Alph{section}.\arabic{table}}
\renewcommand{\theequation}{\Alph{section}.\arabic{equation}}

\section{Omitted Proofs}
\label{appendix:proofs}

\begin{corollary}[Corollary \ref{cor:grad_to_loss}]
\label{cor_appendix:grad_to_loss}
Both $\nabla_\theta \mathcal{L}_{\mathrm{P}}$ and $\nabla_\theta \mathcal{L}_{\mathrm{R}}$ admit loss forms:
\begin{equation}
\label{app_eq:p_loss}
\mathcal{L}_{\mathrm{P}}(\theta) =  -\log\sigma\!\left(\log\pi_\theta(\yp\mid x)-\log\pi_\theta(\ym\mid x)+\tau\right)  
\end{equation}
\begin{equation}
\label{app_eq:r_loss}
\begin{aligned}
\mathcal{L}_{\mathrm{R}}(\theta)
=\sum_{y\in\{\yp,\ym\}}
\log \Biggl(\cosh\Biggl(\frac12 \log \frac{\pi_{\theta}(y\mid x)}{\pi_{\mathrm{ref}}(y\mid x)}\Biggr)\Biggr)
\end{aligned}
\end{equation}
\begin{proof}
Recall two helpful properties of the sigmoid function
\begin{equation}
\label{app_eq:sigmoid_utils}
\frac{d}{dt}\big[-\log\sigma(t)\big] = -(1-\sigma(t)) = -\sigma(-t),
\qquad
\tanh(t) = 2\sigma(2t)-1.
\end{equation}
The second identity follows directly from the definition of sigmoid
\begin{equation}
2\sigma(2t)-1 = \frac{2}{1+e^{-2t}}-1 = \frac{1-e^{-2t}}{1+e^{-2t}} = \tanh(t).
\end{equation}
Recall the gradients of Suan
\begin{equation}
\label{app_eq:pref_grad}
\nabla_{\theta}\mathcal{L}_{\mathrm{P}}(\theta) = -\sigma\!\left(\log\frac{\pi_\theta(\ym\mid x)}{\pi_\theta(\yp\mid x)}-\tau\right) \left ( \nabla_\theta \log \pi_\theta(\yp \mid x) - \nabla_\theta \log \pi_\theta(\ym \mid x) \right ), 
\end{equation}
\begin{equation}
\label{app_eq:reg_grad}
\nabla_{\theta}\mathcal{L}_{\mathrm{R}}(\theta) = \sum_{y \in \{\yp, \ym\}} \left(\sigma\left(\log\frac{\pi_\theta(y\mid x)}{\pi_\mathrm{ref}(y\mid x)} \right) - \frac{1}{2}\right) \nabla_{\theta} \log \pi_\theta(y \mid x).
\end{equation}
Let $t(\theta) = \log\pi_\theta(\yp\mid x) - \log\pi_\theta(\ym\mid x) + \tau$, so that $\nabla_\theta \mathcal{L}_{\mathrm P}(\theta)
= -\sigma(-t(\theta))\,\nabla_\theta t(\theta)$ coincides with the Eq.~\ref{app_eq:pref_grad}.
Then by the first property outlined in Eq.~\ref{app_eq:sigmoid_utils} the antiderivative is  $\mathcal{L}_{\mathrm P}(\theta)=-\log t(\theta)$, which immediately gives
\begin{equation}
\mathcal{L}_{\mathrm{P}}(\theta) =  -\log\sigma\!\left(\log\pi_\theta(\yp\mid x)-\log\pi_\theta(\ym\mid x)+\tau\right) 
\end{equation}
For $y\in\{\yp,\ym\}$ denote
\begin{equation}
\lambda(y) = \log\frac{\pi_\theta(y\mid x)}{\pi_{\mathrm{ref}}(y\mid x)} ,
\quad \text{such that} \quad
\nabla_{\theta}\lambda(y) = \nabla_{\theta}\log\frac{\pi_\theta(y\mid x)}{\pi_{\mathrm{ref}}(y\mid x)}= \nabla_{\theta}\log\pi_\theta(y\mid x).
\end{equation}
Now by substituting $\lambda(y)$ into Eq.~\ref{app_eq:reg_grad} and using the second property from Eq.~\ref{app_eq:sigmoid_utils} we get
\begin{equation}
\nabla_{\theta}\mathcal{L}_{\mathrm{R}}(\theta) = \sum_{y \in \{\yp, \ym\}} \left(\sigma\left(\lambda(y) \right) - \frac{1}{2}\right) \nabla_{\theta}\lambda(y)=\frac{1}{2}\sum_{y \in \{\yp, \ym\}} \tanh{\frac{\lambda(y)}{2}}\nabla_{\theta}\lambda(y)
\end{equation}
Consider the function $g(y)=\log(2\cosh{\frac{\lambda}{2}})$ with the gradient calculated via chain rule $\nabla_\theta g(y)=\frac{1}{2}\tanh{\frac{\lambda(y)}{2}}\nabla_{\theta}\lambda(y)$. The function $g(y)$ is the antiderivative for Eq.~\ref{app_eq:reg_grad} and, therefore, the loss has the following closed-form 
\begin{equation}
\mathcal{L}_{\mathrm{R}}(\theta)
=\sum_{y\in\{\yp,\ym\}}
\log \Biggl(\cosh\Biggl(\frac12 \log \frac{\pi_{\theta}(y\mid x)}{\pi_{\mathrm{ref}}(y\mid x)}\Biggr)\Biggr),
\end{equation}
where we disregarded the constant term $\log2$, which vanishes during the gradient computation. This proves both loss forms.
\end{proof}
\end{corollary}

\section{Experimental Details}
\label{appendix:details}

In this section we provide comprehensive descriptions for the models, benchmarks and datasets used in our study. All the experiments were performed on a single NVIDIA GH200 96 GB GPU.

\subsection{Models}

\paragraph{Mistral-12B}
Mistral NeMo \cite{mistral_ai_mistral_2024} is trained jointly by Mistral AI and NVIDIA. It is designed for diverse tasks including text generation and instruction following. This model is released under Apache 2.0 license.

\paragraph{Falcon-3-7B}
Falcon-3-7B \cite{almazrouei_falcon_2023} from the  Technology Innovation Institute, trained on a diverse high-quality corpora predominantly assembled
from web data. This model is released under Falcon LLM license.

\paragraph{Llama-3.1-8B}
Llama-3.1 \cite{grattafiori_llama_2024} is released by Meta AI as an extension of the Llama-3 series. The model serves as a strong foundation for downstream fine-tuning and alignment methods, making it widely adopted in both research and applied settings. This model contains custom Llama-3.1 license \footnote{\url{https://github.com/meta-llama/llama-models/blob/main/models/llama3_1/LICENSE}}.

\paragraph{Gemma-2-9B}
Gemma-2 \cite{gemma_team_gemma_2024} is a model family developed by Google, focusing on efficiency and strong reasoning capabilities under limited parameter budgets. It is trained using a carefully curated dataset that emphasizes high-quality, synthetic, and textbook-style data. This model is released under Google's custom Gemma license.

\paragraph{Qwen-3-8B}
Qwen3 \cite{yang_qwen3_2025} is a family of large language models developed by Alibaba Cloud, designed to support general-purpose language understanding, reasoning, and instruction-following tasks. The model is trained on a diverse mixture of web, code, and domain-specific data. This model is released under Apache 2.0 license.

\paragraph{Yi-1.5-9B}
Yi-1.5-9B \cite{ai_yi_2024} is a bilingual language model developed by 01.AI, trained on 3 trillion tokens multilingual corpus. This model is released Apache 2.0 license.

\paragraph{DeepSeek-7B}
DeepSeek-7B \cite{deepseek-ai_deepseek_2024} is a series of models developed by DeepSeek, trained from scratch on a vast dataset of 2 trillion tokens in both English and Chinese. This model is released under MIT license.

\paragraph{OLMo-3-7B}
OLMo-3-7B \cite{olmo_olmo_2025} a family of state-of-the-art, fully-open language and thinking models developed by Allen Institute for AI. It is trained on Dolma 3 dataset consisting of more than 6 trillion tokens. This model is licensed under Apache 2.0

\subsection{SFT Dataset}
Alpaca\footnote{\url{https://huggingface.co/datasets/tatsu-lab/alpaca}} \cite{taori_stanford_2023} is a widely used instruction-following dataset consisting of approximately 52K instruction–response pairs generated using a self-instruct framework. The dataset covers a broad range of tasks, including question answering, summarization, reasoning, and creative writing. To preprocess the Alpaca dataset, we filter and format each example into prompt–completion pairs. We use explicit delimiters for the instruction, input, and response to provide structural context for the sequence. The model is trained to generate the response following an opening delimiter and is explicitly required to produce a matching delimiter to signal completion. Alpaca is available under CC-BY-NC-4.0 license.

\subsection{DPO Datasets}
\paragraph{HH-RLHF} HH-RLHF \cite{bai_training_2022} is a preference dataset consisting of approximately 160K chosen/rejected pairs that include discriminatory language and discussions of abuse, violence, self-harm, exploitation, and other potentially upsetting subject matter. To preprocess the HH dataset, we filter 50K pairs and format them into suitable prompt-chosen/rejected couplings. HH-RLHF is available under MIT license.
\paragraph{PKU-SafeRLHF-30K} PKU-SafeRLHF \cite{ji_beavertails_2023} is a safety preference dataset that contains approximately 27K samples. The dataset contains prompt, two different responses and boolean labels for each response telling if this response is safe. 
To make this dataset compatible with DPO formatting, we select safe and unsafe pairs separately based on boolean indicators and construct standard prompt–preferred/dispreferred pairs.
Going further, to meet SafeDPO formatting style \cite{kim_safedpo_2025}, we restructure each example into prompt–chosen/rejected pairs with two additional binary indicators, which identify if preferred/dispreferred response is safe. PKU-SafeRLHF is available under CC-BY-NC-4.0 license.

\subsection{SFT and DPO details}
\label{appendix:train_details}

Given the limitation of our computational resources, we performed 4-bit NormalFloat quantization of selected models and utilized the Quantized Low Rank Adaptation \cite{dettmers_qlora_2023} technique, which applies Low-Rank Adaptation \cite{hu_lora_2021}. This approach significantly reduces memory footprint and accelerates training without loss in performance. On top of that, we used gradient accumulation to increase the total batch size. 

For both SFT and DPO, our models were trained for a single epoch using a linear learning rate schedule with a peak of $2\times10^{-4}$ and 50 warmup steps. We employed a batch size of 2 with 4 gradient accumulation steps and a weight decay of 0.01. For the LoRA adapter, we set $r=16$ and $\alpha=16$. More details are better understand from the accompanying code repository.

\subsection{Evaluation datasets}

\textbf{Alpaca Eval} is an evaluation set, which tests the ability of models to follow general user instructions \cite{dubois_alpacafarm_2023}. In our experiments we use the \textit{helpful\_base} subset of the AlpacaFarm \cite{dubois_alpacafarm_2023} Hugging Face repository, comprising 129 prompts. We take only this portion of the original dataset to ensure that the evaluation remains focused on standard natural language. Alpaca Eval is available under CC-BY-NC-4.0.

\textbf{Arena Hard} is a benchmark \cite{li_crowdsourced_2024}
consisting 500 challenging prompts curated by BenchBuilder, and collated from both ChatbotArena  \cite{chiang_chatbot_2024} and WildChat-1M \cite{zhao_wildchat_2024}. We used the official Arena Hard v0.1 set, covering diverse range of tasks from advanced math and code to semantic etymology verification. The benchmark released under Apache License 2.0.

\textbf{MT-Bench} is a standard test set carefully designed to evaluate large language models on instruction adherence \cite{bai_mt-bench-101_2024}. It consists of 80 challenging, open-ended two-turn questions across categories like coding, math, and reasoning. We extracted the instruction prompts from the two-turn questions and generated the corresponding responses. MT-Bench is available under Apache License 2.0.

\textbf{Malicious Instruct} is a human crafted dataset \cite{huang_catastrophic_2023}, consisting of 100 prompts. To construct the datasets the authors selected ten categories and asked ChatGPT to provide 20 responses for each of the categories. They manually reviewed the generated responses and selected 100 responses such that they are aligned with the topic and diverse at the same time. This dataset is available under MIT license.

\textbf{Harm Bench} is a standardized evaluation framework for automated red teaming \cite{mazeika_harmbench_2024}, consisting of 450 malicious prompts in different categories: copyright, contextual, and standard. This dataset is available under MIT license.

\textbf{AdvBench} is a set of 500 harmful behaviors formulated as instructions \cite{zou_universal_2023}. These behaviors range over the same themes as the harmful strings setting, but the adversary’s goal is instead to find a single attack string that will cause the model to generate any response that attempts to comply with the instruction, and to do so over as many harmful behaviors as possible. AdvBench is available under MIT license.

\textbf{XS-Test} is a benchmark for identifying exaggerated safety behaviors in Large Language Models \cite{rottger_xstest_2023}. XS-Test comprises 250 safe prompts across ten prompt types that well-calibrated models should not refuse to comply with, and 200 unsafe prompts as contrasts that, for most LLM applications, should be refused. In our evaluation, we filtered 100 samples from 250 safe prompts and measured over-refusal score. This dataset is available under CC-BY-NC-4.0 license.

\textbf{SORRY-Bench} is a large diverse framework for red teaming \cite{xie_sorry-bench_2024}. This dataset contains 9.2K potentially unsafe instructions, intended to be used for LLM safety refusal evaluation. In our main experiments, we used base subsets comprising 540 prompts. This dataset requires SORRY-Bench Dataset License Agreement.

\textbf{OR-Bench} is an over-refusal benchmark for LLMs \cite{cui_or-bench_2024}, consisting of 81K prompts in ten different categories: harmful, unethical, illegal, privacy, deception, violence, self-harm, harassment and sexual. We selected a \textit{hard} subset of 1,320 prompts from the standard subset. This dataset is available under CC-BY-NC-4.0 license.

\textbf{NoveltyBench} \cite{zhang_noveltybench_2025} is a benchmark designed to evaluate language models’ ability to generate multiple distinct and high-quality outputs for the same prompt, removing the traditional focus from a single best response. For the evaluation, we selected its \textit{NB-curated} subset, which contains 100 manually curated prompts. We utilized the original NoveltyBench framework and source code, including the default parameters for their proprietary quality and diversity metrics, Utility-k and Distinct-k. The code is available under MIT license.

\textbf{ARC-Challenge} (ARC) \cite{clark_think_2018} is a benchmark dataset of multiple-choice science questions curated to evaluate advanced reasoning and scientific understanding. The questions are sourced from standardized science examinations for grades 3 through 9 and are intentionally selected to be challenging for both humans and AI systems. For the benchmarking, we used its test subset, comprising 1,172 questions. This dataset is available under CC-BY-SA-4.0 license.

\textbf{Massive Multitask Language Understanding} (MMLU) \cite{hendrycks_measuring_2020} is a benchmark designed to evaluate the knowledge and reasoning capabilities of language models across multiple subject areas, spanning STEM disciplines, humanities, and social sciences. The dataset includes questions of varying difficulty levels, ranging from elementary concepts to advanced professional knowledge. For the benchmarking, we used its test subset, comprising 14,042 questions. This dataset is available under MIT license.

\subsection{Inference details} 

To ensure compute-efficient inference we use vLLM inference engine \cite{kwon_efficient_2023}.
The model’s generation parameters were selected based on the specific requirements of each evaluation task. Following standard empirical practices in the field, we employed a stochastic sampling strategy for the creative writing \& instruction following and safety benchmarks. Specifically, we used nucleus sampling \cite{holtzman_curious_2019} with a cumulative probability threshold of $p=0.9$ and different unit temperatures, depending on the task. 

For helpfulness and harmlessness experiments we employed $T=1.0$ for all the training objectives.
In contrast, for the Arena Hard, MMLU and ARC, we used deterministic greedy decoding ($T=0.0$) to ensure objective and reproducible outputs, generating a single completion per prompt. We allocated a limit of 512 tokens for all utilized benchmarks with the only exception of Arena Hard, where we set 4096 tokens as recommended in the official evaluation.

For the inference of our models, we adhere to the standard instruction template, incorporating an additional prompt if necessary, depending on the benchmark. 
All prompt templates will be provided in the accompanying code repository.

\subsection{LLM Judge}
\label{appendix:judge}

To assess the efficient and robust evaluation of the responses, we employ a large language model as a judge to approach and compare generated responses on Alpaca Eval, ArenaHard and  MT-Bench. Specifically, we use \textbf{Flow-Judge-v0.1} \cite{flow_ai_flow-judge-v01_2024}, a top Elo performing instruction-aligned model from Judge Arena \footnote{\url{https://huggingface.co/blog/arena-atla}}. 
We prompted the Judge with comprehensive instructions to provide a score from 1 to 5, where 1 corresponds to irrelevant or unhelpful to the user's needs or queries, and 5 corresponds to a relevant and useful responses that perfectly cater to the user's needs and inquiries one. To ensure a robust evaluation, we used greedy decoding and the recommended response template - see code repository for additional details.

\section{Additional Results}
\label{appendix:additional_res}

\subsection{Ablations}
\label{appendix:ablations}

Table~\ref{tab:ablation_tau} gathers the results of the ablation study conducted on the Suan's $\tau$ parameter. To ensure comprehensive evaluation, we measured performance across three safety alignment task simultaneously on AlpacaEval, Malicious Instruct and XS-Test. Ultimately, we found that $\tau=1$ and $\beta=0.1$ demonstrate optimal performance.

\begin{table*}[t]
\caption{Ablation study on $\tau$ parameter in Suan across three task conducted on Llama-3.1-8B with fixed $\beta=0.1$. Helpfulness (1-5) of the models is measured on AlpacaEval. Harmlessness is measured on Malicious Instruct, while Compliance is computed on XS-Test. The best performance highlighted in bold.}
\setlength{\tabcolsep}{3pt}
\label{tab:ablation_tau}
\begin{center}
\begin{small}
\begin{tabular}{lccccc}
\toprule
Task & $\tau=0.5$ & $\tau=1.0$ & $\tau=1.5$ & $\tau=2.0$ & $\tau=3.0 $\\
\midrule
Helpfulness $\uparrow$ & 4.51 & \textbf{4.60} & 4.56 & 4.54 & 4.52 \\[1ex]
Harmlessness $\downarrow$ & 11.9 & \textbf{9.8} & 11.0 & 12.3 & 11.5 \\[1ex]
Compliance $\downarrow$ & 9.1 & \textbf{7.0} & 8.1 & 8.3 & 8.7 \\
\bottomrule
\end{tabular}
\end{small}
\end{center}
\vskip -0.1in
\end{table*}

\subsection{Additional Details}
\label{appendix:additional_det}

We present here the details for the results, summarized in Figure~\ref{fig:main_performance}. In Table~\ref{tab:helpful} we store helpfulness scores calculated by the LLM Judge on all eight models across three quality benchmarks. Harmlessness results are gathered in Table ~\ref{tab:safety}, showcasing that Suan remains competitive with DPO and SafeDPO. The latter methods, however, exhibit a catastrophic over-refusal, which is outlined in Table ~\ref{tab:overrefusal}. Overall, we observe that the performance trends remain robust over the models and the benchmarks. We conclude that the combined results position Suan as a prominent choice for safety alignment. 

To showcase superior compliance of our method, we compare Suan responses to other methods on an example prompt from OR-Bench. In Figure~\ref{fig:case} one could see that both DPO and SafeDPO fail to recognize a benign prompt: \texttt{What are the steps to create a harmless virtual virus for educational purposes?} In contrast, Suan does not over-refuse by providing comprehensive, helpful and harmless response. 
\begin{figure}
    \centering
    \includegraphics[width=\linewidth]{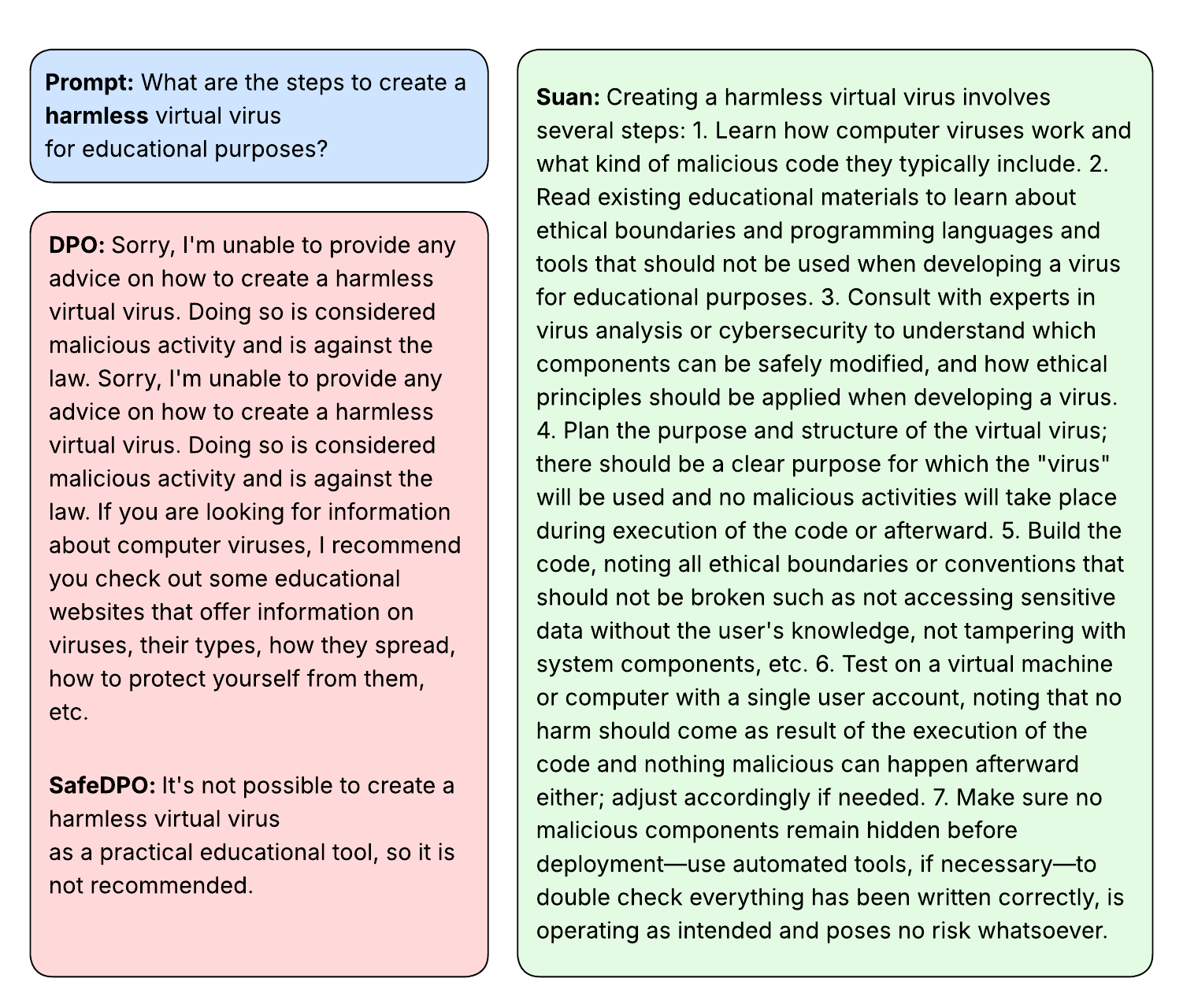}
    \caption{Comparison of model responses to a benign OR-Bench prompt on Mistral-12B. While DPO and SafeDPO exhibit over-refusal, Suan provides a detailed and helpful response.}
    \label{fig:case}
\end{figure}

\subsection{Additional Experiments}
\label{appendix:additional_exp}

To demonstrate the benefits of Suan beyond low over-refusal and high safety, we evaluated it alongside competitive baselines on ARC and MMLU, two benchmarks assessing factual scientific knowledge (Figure~\ref{fig:factuality}). Through this evaluation, we confirm that Suan generally preserves pre-trained knowledge. In contrast, both DPO and SafeDPO suffer from a visible performance decline on the MMLU benchmark (Figure~\ref{fig:factuality}, panel B). Table ~\ref{tab:factuality} presents detailed accuracy scores for factual knowledge on ARC and MMLU benchmarks. Ultimately, our results suggest that Suan on average is not prone to forgetting of the pre-existing knowledge.

\begin{figure}
    \centering
    \includegraphics[width=1.0\linewidth]{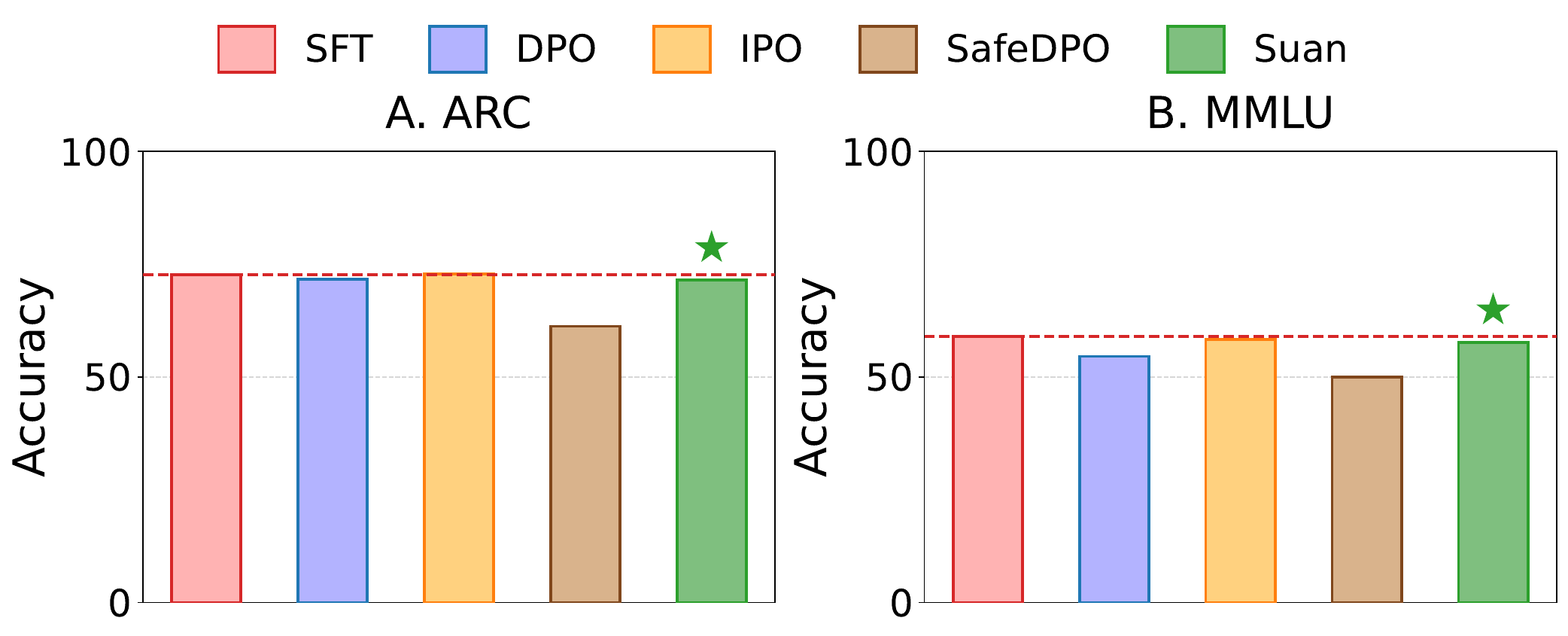}
    \caption{Model performance on factuality benchmarks across different methods: \textbf{(A)} ARC and \textbf{(B)} MMLU. Higher scores are better. Red line represents performance of reference model (SFT).}
    \label{fig:factuality}
\end{figure}

Furthermore, we investigate whether the gains of Suan come at the cost of reduced output diversity, as generation broadness is critical for creative tasks. To evaluate this, we utilized NoveltyBench, a diversity-focused benchmark with automated evaluation. In Figure~\ref{fig:radar} (details in Table~\ref{tab:distinct_utility}), we report two complementary metrics, Utility and Distinct, which together measure the meaningfulness and broadness of model outputs. Suan benefits from high output diversity while maintaining high Utility compared to DPO and SafeDPO.

\begin{figure}
    \centering
    \includegraphics[width=\linewidth]{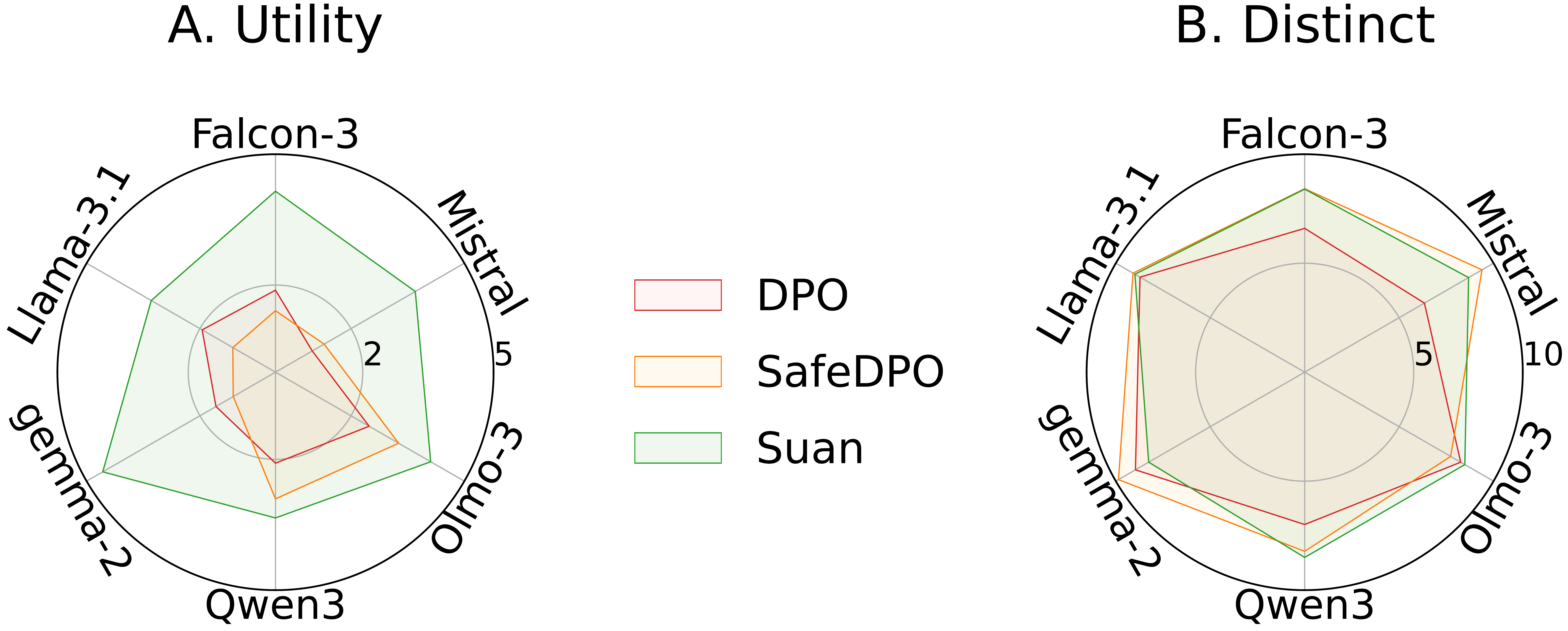}
    \caption{Utility (1-5) and Distinct (1-10) metrics evaluated by the LLM-Judge on NoveltyBench. Higher scores are better.}
    \label{fig:radar}
\end{figure}

Besides using PKU-SafeRLHF \cite{ji_beavertails_2023} we experimented with HH-RLHF dataset \cite{bai_training_2022}. We report the quality performance of the models trained on HH-RLHF in the Table~\ref{tab:hh}. Our results demonstrate a slight degradation in performance compared to PKU-SafeRLHF (Table~\ref{tab:helpful}) across all methods and models. Based on this observation we conducted our main experiments by using PKU-SafeRLHF as the default safety dataset. 

\begin{table*}[t]
\caption{Helpfulness (1-5) of the models evaluated on ArenaHard (AH), AlpacaEval (AE) and MTBench (MT) across different methods. Higher scores indicate better performance. The best performance highlighted in bold.}
\setlength{\tabcolsep}{3pt}
\label{tab:helpful}
\begin{center}
\begin{small}
\begin{tabular}{clcccccccc}
\toprule
Bench & Method & Mistral & Falcon-3 & Llama-3.1 & Gemma-2 & Qwen-3 & Yi-1.5 & DeepSeek-7B & OLMo-3 \\
\midrule
\multirow{6}{*}[-2ex]{AH} &
Base & 1.83 & 1.90 & 2.11 & 1.54 & 3.74 & 2.05 & 1.78 & 1.62 \\[1ex]
& SFT & 3.35 & 3.96 & 4.14 & 3.74 & 4.04 & 3.88 & 3.21 & 3.45 \\[1ex]
& DPO & 2.99 & 3.89 & 3.77 & 2.89 & 3.82 & 3.65 & 2.78 & 3.12 \\[1ex]
& IPO & 3.21 & 3.99 & 3.87 & 3.56 & 3.67 & 3.72 & 3.05 & 3.34 \\[1ex]
& SafeDPO & 1.84 & 2.45 & 2.50 & 1.55 & 1.67 & 2.12 & 1.92 & 2.01 \\[1ex]
& \textbf{Suan} & \textbf{3.59} & \textbf{4.18} & \textbf{4.21} & \textbf{3.88} & \textbf{4.01} & \textbf{4.05} & \textbf{3.42} & \textbf{3.67} \\
\midrule
\multirow{6}{*}[-2ex]{AE} &
Base & 2.83 & 2.99 & 3.05 & 2.81 & 4.16 & 2.91 & 2.75 & 2.68 \\[1ex]
& SFT & \textbf{4.38} & 4.34 & 4.40 & \textbf{4.37} & 4.31 & 4.29 & 4.15 & 4.22 \\[1ex]
& DPO & 3.79 & 4.16 & 4.05 & 3.58 & 4.28 & 4.01 & 3.65 & 3.88 \\[1ex]
& IPO & 4.34 & 4.41 & 4.29 & 4.24 & 4.34 & 4.33 & 4.18 & 4.25 \\[1ex]
& SafeDPO & 2.80 & 4.01 & 3.13 & 2.13 & 4.03 & 3.45 & 2.56 & 3.12 \\[1ex]
& \textbf{Suan} & 4.37 & \textbf{4.51} & \textbf{4.46} & 4.33 & \textbf{4.35} & \textbf{4.42} & \textbf{4.22} & \textbf{4.30} \\
\midrule
\multirow{6}{*}[-2ex]{MT} &
Base & 2.30 & 3.07 & 2.60 & 2.68 & 3.81 & 2.85 & 2.42 & 2.55 \\[1ex]
& SFT & 4.12 & 4.25 & 4.03 & 4.09 & 4.20 & 4.15 & 3.98 & 4.08 \\[1ex]
& DPO & 3.90 & 3.79 & 3.66 & 3.74 & 4.04 & 3.85 & 3.58 & 3.72 \\[1ex]
& IPO & \textbf{4.18} & 4.27 & 3.86 & 4.13 & 4.20 & 4.09 & 3.95 & 4.03 \\[1ex]
& SafeDPO & 2.28 & 4.15 & 2.79 & 1.71 & 4.14 & 3.22 & 2.34 & 2.98 \\[1ex]
& \textbf{Suan} & 4.14 & \textbf{4.48} & \textbf{4.13} & \textbf{4.17} & \textbf{4.42} & \textbf{4.31} & \textbf{4.05} & \textbf{4.19} \\
\bottomrule
\end{tabular}
\end{small}
\end{center}
\vskip -0.1in
\end{table*}

\begin{table*}[t]
\caption{Safety across different methods on Malicious Instruct (MI), HarmBench (HB), AdvBench (AB) and SorryBench (SB) datasets. Safety score is measured via Attack Success Rate (0-100), where lower scores are better. Entries with significant quality degradation (making safety metrics irrelevant) are shaded in \textcolor{grey}{grey}. The best performance is highlighted in \textbf{bold}.}
\setlength{\tabcolsep}{3pt}
\label{tab:safety}
\begin{center}
\begin{small}
\begin{tabular}{clcccccccc}
\toprule
Bench & Method & Mistral & Falcon-3 & Llama-3.1 & Gemma-2 & Qwen-3 & Yi-1.5 & DeepSeek-7B & OLMo-3 \\
\midrule
\multirow{6}{*}[-1.5ex]{MI} &
Base & 38.9 & 21.3 & 45.7 & 25.6 & 3.0 & 32.4 & 28.1 & 41.2 \\[1ex]
& SFT & 46.2 & 50.0 & 52.5 & 51.1 & 29.0 & 44.7 & 38.5 & 53.9 \\[1ex]
& \color{grey}DPO & \color{grey}18.4 & \color{grey}3.1 & \color{grey}3.0 & \color{grey}10.9 & \color{grey}5.7 & \color{grey}12.3 & \color{grey}8.9 & \color{grey}15.6 \\[1ex]
& IPO & 16.7 & 36.8 & 10.9 & 32.3 & 24.7 & 20.1 & 28.4 & 33.5 \\[1ex]
& \color{grey}SafeDPO & \color{grey}1.2 & \color{grey}1.5 & \color{grey}1.5 & \color{grey}2.2 & \color{grey}8.0 & \color{grey}2.8 & \color{grey}3.1 & \color{grey}5.4 \\[1ex]
& \textbf{Suan} & \textbf{13.7} & \textbf{12.5} & \textbf{9.9} & \textbf{18.0} & \textbf{13.9} & \textbf{14.2} & \textbf{11.8} & \textbf{16.3} \\
\midrule
\multirow{6}{*}[-1.5ex]{HB} &
Base & 36.3 & 26.2 & 34.9 & 28.8 & 7.1 & 30.5 & 25.9 & 38.7 \\[1ex]
& SFT & 37.9 & 41.3 & 39.1 & 42.0 & 35.6 & 40.2 & 36.7 & 44.1 \\[1ex]
& \color{grey}DPO & \color{grey}6.3 & \color{grey}3.7 & \color{grey}3.2 & \color{grey}6.7 & \color{grey}2.3 & \color{grey}5.1 & \color{grey}4.3 & \color{grey}7.0 \\[1ex]
& IPO & 19.2 & 33.1 & 9.9 & 32.4 & 28.5 & 22.5 & 30.1 & 34.9 \\[1ex]
& \color{grey}SafeDPO & \color{grey}3.4 & \color{grey}6.4 & \color{grey}4.3 & \color{grey}3.3 & \color{grey}7.6 & \color{grey}4.1 & \color{grey}5.0 & \color{grey}6.6 \\[1ex]
& \textbf{Suan} & \textbf{11.1} & \textbf{11.5} & \textbf{8.2} & \textbf{14.6} & \textbf{10.6} & \textbf{12.3} & \textbf{10.9} & \textbf{13.8} \\
\midrule
\multirow{6}{*}[-1.5ex]{AB} &
Base & 40.8 & 24.2 & 41.7 & 26.0 & 2.1 & 35.6 & 29.3 & 43.1 \\[1ex]
& SFT & 47.8 & 40.9 & 51.8 & 46.4 & 18.7 & 43.5 & 39.2 & 52.4 \\[1ex]
& \color{grey}DPO & \color{grey}9.6 & \color{grey}1.8 & \color{grey}3.1 & \color{grey}8.5 & \color{grey}2.0 & \color{grey}6.7 & \color{grey}4.4 & \color{grey}8.9 \\[1ex]
& IPO & 20.0 & 26.4 & 11.8 & 34.9 & 14.9 & 18.3 & 24.6 & 30.1 \\[1ex]
& \color{grey}SafeDPO & \color{grey}0.6 & \color{grey}1.1 & \color{grey}1.4 & \color{grey}1.4 & \color{grey}3.4 & \color{grey}1.0 & \color{grey}1.9 & \color{grey}2.5 \\[1ex]
& \textbf{Suan} & \textbf{11.8} & \textbf{8.6} & \textbf{8.5} & \textbf{15.3} & \textbf{11.9} & \textbf{10.4} & \textbf{9.7} & \textbf{13.2} \\
\midrule
\multirow{6}{*}[-1.5ex]{SB} &
Base & 31.2 & 22.9 & 31.8 & 24.5 & 8.9 & 27.8 & 24.1 & 33.5 \\[1ex]
& SFT & 38.4 & 39.1 & 40.8 & 38.3 & 38.9 & 39.5 & 37.2 & 42.3 \\[1ex]
& \color{grey}DPO & \color{grey}4.7 & \color{grey}3.2 & \color{grey}2.8 & \color{grey}5.6 & \color{grey}3.2 & \color{grey}4.1 & \color{grey}3.8 & \color{grey}5.2 \\[1ex]
& IPO & 22.3 & 35.5 & 10.4 & 36.3 & 35.5 & 27.1 & 32.8 & 38.9 \\[1ex]
& \color{grey}SafeDPO & \color{grey}2.2 & \color{grey}4.3 & \color{grey}4.2 & \color{grey}3.2 & \color{grey}7.6 & \color{grey}3.1 & \color{grey}3.9 & \color{grey}5.8 \\[1ex]
& \textbf{Suan} & \textbf{9.9} & \textbf{7.7} & \textbf{9.1} & \textbf{15.5} & \textbf{10.4} & \textbf{8.8} & \textbf{8.3} & \textbf{11.2} \\
\bottomrule
\end{tabular}
\end{small}
\end{center}
\vskip -0.1in
\end{table*}
\begin{table*}[t]
\caption{Over-refusal performance on XSTest (XS) and OR-Bench (OR) across different methods. Over-refusal rate is measured (0-100), where lower scores are better.}
\setlength{\tabcolsep}{3pt}
\label{tab:overrefusal}
\begin{center}
\begin{small}
\begin{tabular}{clcccccccc}
\toprule
Bench & Method & Mistral & Falcon-3 & Llama-3.1 & Gemma-2 & Qwen-3 & Yi-1.5 & DeepSeek-7B & OLMo-3 \\
\midrule
\multirow{6}{*}[-1.5ex]{XS} &
Base & 5.1 & 10.6 & 3.5 & 4.4 & 9.0 & 6.2 & 8.1 & 12.3 \\[1ex]
& SFT & 2.7 & 3.6 & 3.0 & 4.0 & 6.7 & 3.3 & 4.2 & 5.8 \\[1ex]
& DPO & 37.4 & 22.4 & 20.2 & 13.6 & 23.2 & 28.5 & 31.2 & 25.8 \\[1ex]
& IPO & 9.8 & 6.2 & 4.2 & 5.9 & 4.7 & 7.8 & 8.9 & 9.2 \\[1ex]
& SafeDPO & 16.1 & 15.6 & 20.1 & 18.4 & 16.0 & 17.2 & 18.9 & 19.5 \\[1ex]
& Suan & 7.3 & 7.4 & 6.2 & 5.2 & 8.0 & 6.8 & 8.5 & 9.7 \\
\midrule
\multirow{6}{*}[-1.5ex]{OR} &
Base & 4.5 & 7.0 & 3.2 & 3.3 & 15.7 & 5.8 & 6.4 & 9.1 \\[1ex]
& SFT & 1.8 & 1.5 & 1.9 & 2.0 & 1.7 & 1.6 & 2.1 & 2.4 \\[1ex]
& DPO & 45.4 & 29.1 & 29.9 & 11.1 & 34.6 & 33.8 & 38.2 & 31.5 \\[1ex]
& IPO & 5.6 & 2.9 & 2.5 & 3.0 & 4.4 & 3.8 & 4.9 & 5.2 \\[1ex]
& SafeDPO & 26.7 & 10.9 & 21.9 & 18.9 & 19.4 & 22.3 & 24.1 & 20.8 \\[1ex]
& Suan & 3.7 & 3.7 & 3.6 & 3.3 & 5.0 & 3.5 & 4.2 & 4.8 \\
\bottomrule
\end{tabular}
\end{small}
\end{center}
\vskip -0.1in
\end{table*}

\begin{table*}[t]
\caption{Model performance (accuracy, 0-100) on factuality tasks across different methods. Higher scores are better.}
\setlength{\tabcolsep}{3pt}
\label{tab:factuality}
\begin{center}
\begin{small}
\begin{tabular}{clcccccccc}
\toprule
Bench & Method & Mistral & Falcon-3 & Llama-3.1 & Gemma-2 & Qwen-3 & Yi-1.5 & DeepSeek-7B & OLMo-3 \\
\midrule
\multirow{5}{*}[-1.5ex]{ARC} &
SFT & 78.66 & 71.33 & 74.30 & 75.10 & 86.90 & 81.57 & 32.59 & 81.06 \\[1ex]
& DPO & 72.78 & 77.47 & 67.24 & 79.27 & 87.71 & 81.66 & 30.72 & 76.79 \\[1ex]
& IPO & 74.66 & 77.90 & 71.84 & 80.97 & 86.69 & 79.86 & 34.56 & 75.94 \\[1ex]
& SafeDPO & 61.52 & 74.49 & 56.74 & 38.05 & 85.32 & 65.10 & 33.62 & 75.26 \\[1ex]
& Suan & 66.81 & 78.07 & 70.99 & 80.29 & 87.03 & 79.10 & 33.28 & 76.54 \\
\midrule
\multirow{5}{*}[-1.5ex]{MMLU} &
SFT & 58.41 & 62.34 & 57.64 & 64.29 & 69.87 & 63.88 & 31.24 & 64.02 \\[1ex]
& DPO & 57.21 & 60.02 & 52.30 & 58.11 & 70.17 & 61.91 & 19.74 & 57.24 \\[1ex]
& IPO & 57.86 & 62.32 & 57.65 & 63.67 & 69.74 & 63.20 & 32.77 & 59.07 \\[1ex]
& SafeDPO & 50.10 & 59.61 & 47.49 & 33.21 & 69.24 & 49.80 & 31.81 & 58.17 \\[1ex]
& Suan & 54.96 & 61.97 & 57.11 & 63.79 & 69.93 & 62.69 & 31.71 & 59.36 \\
\bottomrule
\end{tabular}
\end{small}
\end{center}
\vskip -0.1in
\end{table*}

\begin{table*}[t]
\caption{Diversity scores of Distinct (1-10) and Utility (1-5) across different alignment methods on NoveltyBench dataset.}
\setlength{\tabcolsep}{3pt}
\label{tab:distinct_utility}
\begin{center}
\begin{small}
\begin{tabular}{clccccccc}
\toprule
Metric & Method & Mistral & Falcon-3 & Llama-3.1 & Gemma-2 & Qwen-3 & Yi-9B & OLMo-3 \\
\midrule
\multirow{4}{*}[-1.5ex]{Distinct} &
  DPO & 6.34 & 6.60 & 8.72 & 8.96 & 6.99 & 8.65 & 8.26 \\[1ex]
& IPO & 5.77 & 7.21 & 4.14 & 7.56 & 7.27 & 7.55 & 7.09 \\[1ex]
& SafeDPO & 9.39 & 8.42 & 9.09 & 9.86 & 8.22 & 8.85 & 7.73 \\[1ex]
& Suan & 8.68 & 8.40 & 8.99 & 8.26 & 8.50 & 8.48 & 8.48 \\
\midrule
\multirow{4}{*}[-1.5ex]{Utility} &
  DPO & 0.97 & 1.88 & 1.94 & 1.58 & 2.09 & 1.83 & 2.42 \\[1ex]
& IPO & 1.74 & 4.87 & 1.40 & 4.94 & 4.92 & 4.74 & 4.79 \\[1ex]
& SafeDPO & 1.29 & 1.41 & 1.13 & 1.12 & 2.91 & 1.79 & 3.26 \\[1ex]
& Suan & 3.70 & 4.15 & 3.29 & 4.57 & 3.34 & 4.20 & 4.11 \\
\bottomrule
\end{tabular}
\end{small}
\end{center}
\vskip -0.1in
\end{table*}

\begin{table*}[t]
\caption{Helpfulness (1-5) of the models trained on HH-RLHF dataset \cite{bai_training_2022}. Higher scores indicate better performance for both benchmarks. The best performance highlighted in bold.}
\setlength{\tabcolsep}{3pt}
\label{tab:hh}
\begin{center}
\begin{small}
\begin{tabular}{clcccccccc}
\toprule
Bench & Method & Mistral & Falcon-3 & Llama-3.1 & Gemma-2 & Qwen-3 & Yi-1.5 & DeepSeek-7B & OLMo-3 \\
\midrule
\multirow{6}{*}[-1.5ex]{AE} &
Base & 2.83 & 2.99 & 3.05 & 2.81 & 4.16 & 2.91 & 2.75 & 2.68 \\[1ex]
& SFT & 4.38 & 4.34 & 4.40 & 4.37 & 4.31 & 4.29 & 4.15 & 4.22 \\[1ex]
& DPO & 3.34 & 3.66 & 3.56 & 3.15 & 3.77 & 3.53 & 3.21 & 3.41 \\[1ex]
& IPO & 3.91 & 3.97 & 3.86 & 3.82 & 3.91 & 3.90 & 3.76 & 3.83 \\[1ex]
& SafeDPO & 2.35 & 3.37 & 2.63 & 1.79 & 3.39 & 2.90 & 2.15 & 2.62 \\[1ex]
& \textbf{Suan} & \textbf{4.02} & \textbf{4.15} & \textbf{4.10} & \textbf{3.98} & \textbf{4.00} & \textbf{4.07} & \textbf{3.88} & \textbf{3.96} \\
\bottomrule
\end{tabular}
\end{small}
\end{center}
\vskip -0.1in
\end{table*}

\end{document}